\documentclass[conference]{IEEEtran}
\IEEEoverridecommandlockouts
\usepackage{mathtools}
\usepackage{amsmath,amssymb,amsfonts}
\usepackage{algorithm}
\usepackage{amsthm}
\usepackage{cancel}
\usepackage{algorithmic}
\usepackage{hyperref}
\usepackage[square,sort,comma,numbers]{natbib}

\newtheorem*{remark}{Remark}

\usepackage{dsfont}

\usepackage{graphicx}
\usepackage{textcomp}
\newtheorem{theorem}{Theorem}
\newtheorem{lemma}[theorem]{Lemma}
\usepackage{xcolor}
\def\BibTeX{{\rm B\kern-.05em{\sc i\kern-.025em b}\kern-.08em
    T\kern-.1667em\lower.7ex\hbox{E}\kern-.125emX}}
\begin{document}

\title{Online Reinforcement Learning for Periodic MDP\\ 
\thanks{The authors are with the Department of Electrical Engineering
IIT Delhi, . Email: \{Ayush.Aniket, arpanc\}@ee.iitd.ac.in . }
\thanks{The work of A.C. was supported by the professional development fund and professional development allowance at IIT Delhi, and the following grants: (i) grant no. RP04215G from I-Hub Foundation on Cobotics, and (ii) grant no. MI02266 through the MFIRP scheme at IIT Delhi.}
}

\author{\IEEEauthorblockN{ Ayush Aniket}
\and
\IEEEauthorblockN{ Arpan Chattopadhyay}}

\maketitle

\begin{abstract}
We study learning in periodic Markov Decision Process(MDP), a special type of non-stationary MDP where both the state transition probabilities and reward functions vary periodically, under the average reward maximization setting. We formulate the problem as a stationary MDP by augmenting the state space with the period index, and propose  a periodic upper confidence bound  reinforcement learning-2 (PUCRL2) algorithm. We show that the  regret of PUCRL2 varies linearly with the period $N$ and as $\sqrt{Tlog T}$ with the horizon length $T$. Numerical results demonstrate the efficacy of PUCRL2.
\end{abstract}


\section{Introduction}
Reinforcement learning (RL) deals with the problem of optimal sequential decision making in an unknown environment. Sequential decision making in an environment with an unknown statistical model  is typically modeled as a Markov decision process (MDP) where the decision maker, at each time step, has to take an action $a_t$ based on the state $s_t$ of the environment, resulting to a probabilistic  transition to the next state $s_{t+1}$ and a reward $r_t$ accrued by the decision maker depending on the current state and current action.  RL has widespread applications in many areas including robotics \cite{kober2013reinforcement}, resource allocation in wireless networks \cite{5137416}, healthcare \cite{gottesman2019guidelines}, finance \cite{bacoyannis2018idiosyncrasies} etc.

In a stationary MDP, the unknown transition probabilities and reward functions are invariant with time. However, the ubiquitous presence of non-stationarity in real world scenarios often limits the application of stationary reinforcement learning algorithms. Most of the existing works require information about the maximum possible amount of changes that occur in the environment via variation budget in the transition and reward function, or via the number of times the environment changes; this does not require any assumption on the nature of non-stationarity in the environment. On the contrary, we consider a periodic MDP whose state transition probabilities and reward functions are unknown but periodic with a known period $N$. In this setting, we propose the PUCRL2 algorithm and analyse its regret.

Non-stationary RL has been extensively studied in varied scenarios \cite{auer2008near,gajane2018sliding,li2019online,ortner2020variational,cheung2020reinforcement,fei2020dynamic,domingues2021kernel,mao2021near,zhou2020nonstationary,touati2020efficient,wei2021non}.The authors of \cite{auer2008near} propose a restart version of the popular UCRL2 algorithm  meant for stationary RL problems, which achieves an $\mathcal{\Tilde{O}}(l^{1/3}T^{2/3})$ regret where $T$ is the number of time steps, under the setting in which the MDP changes at most $l$ number of times. In the same setting \cite{gajane2018sliding} shows that UCRL2 with sliding windows achieves the same regret. In time-varying environment, a more apposite measure for performance of an algorithm is dynamic regret which measures the difference between accumulated reward through online policy and that of the optimal offline non-stationary  policy. This was first analysed in \cite{li2019online} in a solely  reward varying environment. The authors of   \cite{ortner2020variational} propose first variational dynamic regret bound of $\mathcal{\Tilde{O}}(V^{1/3}T^{2/3})$, where $V$ represents the total variation in the MDP. The work of \cite{cheung2020reinforcement} provides the sliding-window UCRL2 with confidence widening, which achieves an
$\mathcal{\Tilde{O}}((B_r+B_p)^{1/4}T^{3/4})$ dynamic regret, where $B_r$ and $B_p$ represent the maximum amount of possible variation in reward function and transition kernel respectively. They also propose a Bandit-over-RL (BORL) algorithm which tunes the UCRL2-based algorithm in the setting of unknown variational budgets. Further, in the model-free and episodic setting,  \cite{wei2021non} propose  policy optimization algorithms and \cite{fei2020dynamic} propose RestartQ-UCB which achieves a dynamic regret bound  of $\mathcal{\Tilde{O}}(\Delta^{1/3}HT^{2/3})$,where $\Delta$ represent the amount of changes in the MDP and H represents the episode length. The paper \cite{domingues2021kernel} studies a kernel based approach for non-stationarity in MDPs with metric spaces. In the linear MDP case, \cite{mao2021near} and \cite{zhou2020nonstationary} provide optimal regret guarantees.
 Finally the authors of \cite{wei2021non} provide a black-box algorithm which turns any (near)-stationary algorithm to work in a non-stationary environment with optimal dynamic regret  $\Tilde{O}(\min{\sqrt{LT},\Delta^{1/3}T^{2/3}})$, where $L$ and $\Delta$ represent the number and amount of changes of the environment,  respectively.
 
 Periodic MDP  has been marginally studied in literature.  The authors of \cite{riis1965discounted}   study it in the discounted reward setting, where a policy-iteration algorithm is proposed. The authors of   \cite{veugen1983numerical} propose the first state-augmentation method for conversion of periodic MDP into a stationary one, and  analyse the performance of various iterative methods for finding the optimal policy. Recently, \cite{hu2014near} derive a corresponding value iteration algorithm suitable for periodic problems in discounted reward case and provide near-optimal bounds for greedy periodic policies. To the best of our knowledge, RL in periodic MDP has not been studied.
 
In this paper, we make the following contributions: 
 \begin{itemize}
     \item  We study a special form of non-stationarity where the unknown reward and transition functions vary periodically with a known period $N$. 
    \item We propose a modification PUCRL2 of UCRL2, which treats the periodic MDP as stationary MDP with augmented state space. We derive a static regret bound which has  a linear dependence on $N$ and sub-linear dependence on $T$. \item Numerical results show that PUCRL2 performs much better against competing algorithms.

\end{itemize}

\section{Problem Formulation}
\label{Problem Formulation}

A discrete time periodic MDP is defined as the tuple  
$( \mathcal{S}, \mathcal{A}, N,\{ P_i\}_{1\leq i\leq N} ,\{r_i\}_{1\leq i\leq N})$. We consider a finite state space $\mathcal{S}$ and a finite action space  $\mathcal{A}$,  with cardinality $S$ and $A$ respectively. For the $i^{th}$ period index, $P_i : \mathcal{S}\times\mathcal{A}\times\mathcal{S} \xrightarrow{} [0,1]$ defines the transition probability function such that  $\mathbf{p}_i(\cdot | s,a)$ is the probability distribution for next state given current state-action pair, for all $(s,a)$ pair and $r_i : \mathcal{S}\times\mathcal{A} \xrightarrow{} [0,1] $ denotes the reward function where $r_i(s,a)$ is the mean reward given current state-action pair, for all  $(s,a)$ pair. The number  $N \geq 2$ represents the period of the MDP such that  $\mathbf{P}_{t+N} = \mathbf{P}_t$ and $\mathbf{r}_{t+N} = \mathbf{r}_t$ for any time index $t \in \{0,1,2,3,\cdots\}$. The horizon length is $T$ and we assume that $T >> N$.

Now, the PMDP can be transformed into a stationary MDP with augmented state-space (henceforth referred as AMDP). In this AMDP, we couple the period index and states together to obtain an augmented state space $\mathcal{S'} =  \mathcal{S} \times \{1,2,...N\}$; if the state of the original MDP is $s$ at time $t$, then the corresponding state in the AMDP will be $(s, ((t-1)  \mod{N})+1)$, where $\mod$ represents the modulo operator. Consequently, the (time-homogeneous) transition probability of the AMDP for current state $s$ and current action $a$ becomes:
\begin{equation*}
p((s',n') | (s,n),a) = \left\{
        \begin{array}{ll}
            0 & n' \neq n+1 \mod{N}\\
            p_n(s'|s,a) & n' = n+1 \mod{N} \\
        \end{array}
    \right.
\end{equation*}

The corresponding mean reward of the AMDP is given by  $r((s,n),a) = r_{n}(s,a)$. Obviously, under any deterministic stationary policy for the AMDP,  each (state,period index) pair can only be visited after $N$ number of time steps. Thus, the PMDP becomes a stationary AMDP with periodic transition matrix as shown in Figure \eqref{fig:AMDP}. Let $\rho^*$ denote the optimal time-averaged (average expected reward over large number of time steps and then taking a Cesaro limit) reward \cite[Section 8.2.1]{puterman2014markov} of the AMDP. In this paper, we seek to develop an RL algorithm so as to minimize the static regret with respect to this optimal average reward $\rho^*$. Let $\pi$ be any generic  policy for the AMDP. Our problem is:
  \begin{equation*} \label{eqn:regret}
  \min_{\pi}  \sum_{t =1}^{T} (\rho^* - \mathop{\mathbb{E}}_{\pi} (r_t((s_t,n_t),a_t)))
 \end{equation*}

\begin{figure} 
  \centering
  \includegraphics[scale=0.5]{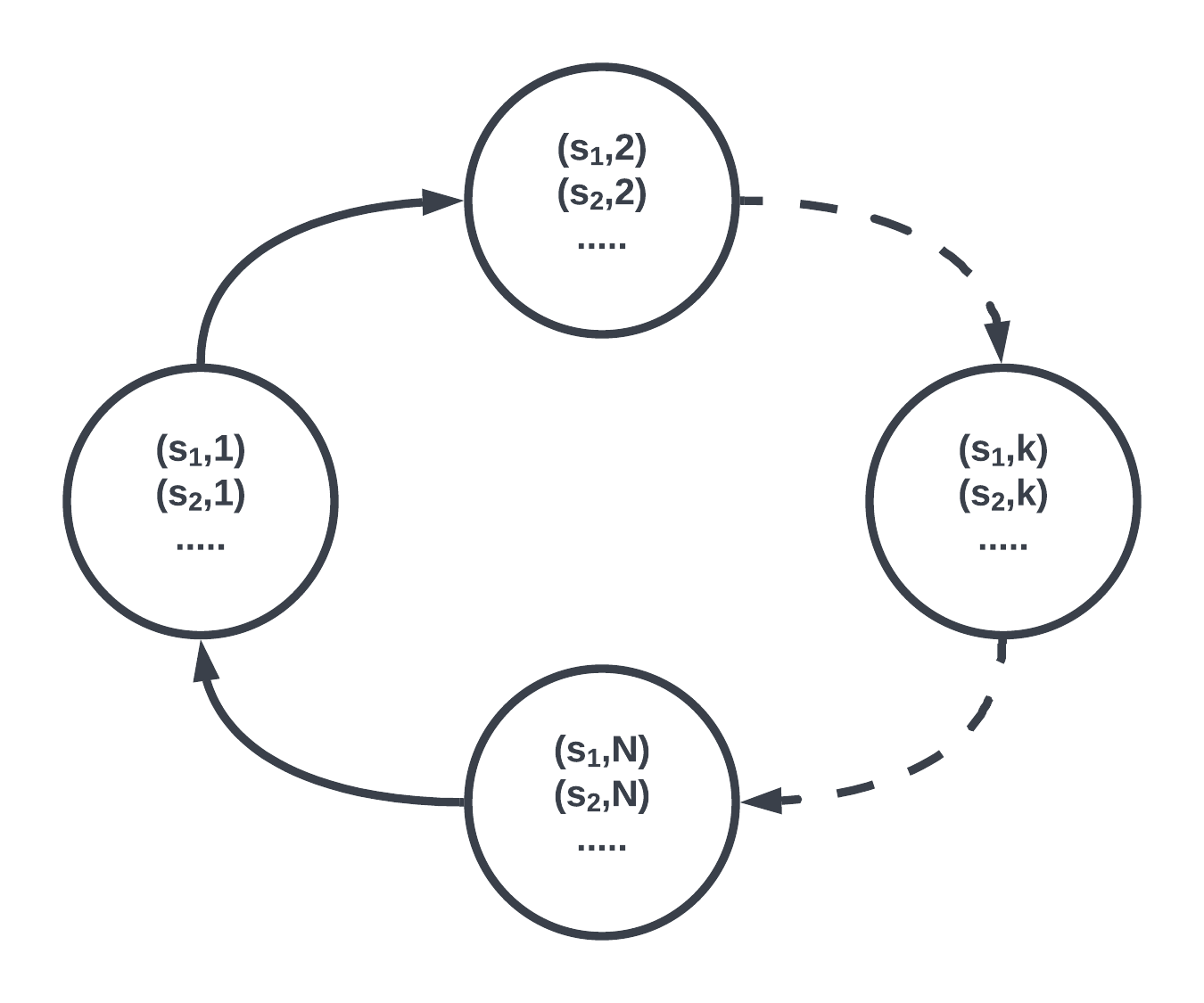}
  \caption{Augmented MDP with periodic states.}
  \label{fig:AMDP}
\end{figure}
\section{The proposed algorithm}
\label{section:algorithms}

In this section, we provide a non-trivial modification to the state of the art UCRL2 algorithm \cite{auer2008near} for PMDP. Our proposed Algorithm~\eqref{alg:PUCRL2} is named as   PUCRL2. PUCRL2  estimates the mean reward and the transition kernel for each augmented state-action pair, while keeping in mind that the transition occurs only to augmented states with the next period index and   the probability of transitioning to other augmented states is zero. Hence the algorithm only estimates the non-zero transition probabilities   $\hat{p}_k((s' | (s,n),a))$ at any time $k$. 

\subsection{PUCRL2 algorithm}
\begin{algorithm}[H]
\caption{P-UCRL2}\label{alg:PUCRL2}
\begin{algorithmic}
\STATE \textbf{Input:} $S,A,N,$ confidence parameter $\delta  \in (0,1) $.\\
\STATE \textbf{Initialization:} $t=1, n = 1 $ \\
\FOR{phase \textit{k} = 1,2,...} 
\STATE{$t_k = t$ \COMMENT {starting time of episode k}}
\STATE{\textbf{1. Initialize episode} \textit{k}:} 
\footnotesize
$v_k((s,n),a)=0$, \\ $n_k((s,n),a) = max\{1,\sum_{\tau =1}^{t-1} {\mathds{1}_{((s_\tau,n_\tau),a_\tau) = ((s,n),a)}\}}$, \\ $n_k((s,n),a,s') = max\{1,\sum_{\tau = 1}^{t-1}{\mathds{1}_{((s_\tau,n_\tau),a_\tau,s_{\tau+1}) = ((s,n),a,s')}\}}$
\newline
\STATE $\hat{p}_k(s'|(s,n),a) = \frac{n_k((s,n),a,s')}{n_k((s,n),a)} \forall{(s,n),a} $ 
\newline
\STATE $ \hat{r}_k((s,n),a) = \frac{\sum_{\tau = 1}^{t-1}(r_\tau\mathds{1}_{((s_\tau,n_\tau),a_\tau) = ((s,n),a)} )}{n_k((s,n),a)} \forall{(s,n),a} $
\normalsize
\vspace{+2mm}
\newline
\STATE{\textbf{2. Update the confidence set}: We define the confidence region for  transition probability function and reward functions as:
\footnotesize
\begin{eqnarray} \label{eqn:confidence_bound_p}
&\mathcal{P}((s,n),a) \coloneqq \{ \mathbf{\Tilde{p}}(\cdot | (s,n),a) :  \nonumber \\ 
& \lVert \mathbf{\Tilde{p}}(\cdot | (s,n),a) - \mathbf{\hat{p}}((\cdot | (s,n),a)) \rVert_1 \leq \sqrt{\frac{14SN \log(2At_k / \delta)}{n_k((s,n),a)}} \}
\end{eqnarray} 
\begin{eqnarray} \label{eqn:confidence_bound_r}
&\mathcal{R}((s,n),a) \coloneqq \{ \Tilde{r}((s,n),a) : \nonumber \\  & \mid \Tilde{r}((s,n),a) - \hat{r}((s,n),a) \mid \leq \sqrt{\frac{7 \log(2SAt_k / \delta)}{2 n_k((s,n),a)}} \}
\end{eqnarray} 
\normalsize
Then, $\mathcal{M}_k$ is the set of all MDP models, such that  \eqref{eqn:confidence_bound_p} and \eqref{eqn:confidence_bound_r} is satisfied for all $((s,n),a)$ pair.}
\STATE{\textbf{3. Optimistic Planning: Compute $(\tilde{M}_k,\tilde{\pi}_k) =$ Modified-Extended Value Iteration \eqref{alg:MEVI}$(\mathcal{M}_k,1/\sqrt{t_k})$}}
\newline
\STATE{\textbf{4. Execute Policies:}}
\WHILE{$v_k(n(s,n),a) < n_k((s,n),a)$}
\STATE Draw $a_t \sim \tilde{\pi}_k$ ; observe reward $r_t$, and the next state $(s_{t+1},n+1)$.
 \STATE Set $v_k((s_t,n_t),a_t) = v_k((s_t,n_t),a_t) +1$ and $t =t+1,  n =((t-1)  \mod{N})+1)$ 
\ENDWHILE
\ENDFOR
\end{algorithmic}
\end{algorithm}

Like UCRL2, PUCRL2 proceeds in episodes. 
At the beginning of each episode, it computes the estimates from previous observations of visits, transitions and rewards accumulated prior to the episode for each (state,period index)-action pair which are stored in $n_k((s,n),a)$, $n_k((s,n),a,s') $ and $\hat{r}_k((s,n),a)$ respectively.
With high probability, the true AMDP lies within a confidence region computed around these estimates as shown in Lemma  \eqref{lemma:probability-of-mdp-outside-confidence-region}. Then PUCRL2 utilizes the confidence bounds as in \eqref{eqn:confidence_bound_p} and \eqref{eqn:confidence_bound_r}, to find an optimistic MDP $\tilde{M}_k$ and policy $\tilde{\pi}_k$ using Modified-EVI Algorithm~\eqref{alg:MEVI} adapted from the extended value iteration (EVI) algorithm depicted in \cite[Section 3.1.2]{auer2008near}. This policy $\tilde{\pi}_k$ is used to take action in the episode until the cumulative number of visits to any  (state,period index) pair gets doubled; similar to the doubling criteria for episode termination of \cite{auer2008near}.

\subsection{Modified-EVI}

Extended value iteration is used in the class of UCRL algorithms to obtain an optimistic MDP model and policy from a high probability confidence region.
According to the convergence criteria of Extended Value Iteration as in \cite[Section 3.1.3]{auer2008near},  aperiodicity is essential i.e. the algorithm should not choose a policy with periodic transition matrix. 
However, as discussed in Section~\eqref{Problem Formulation}, the AMDP is periodic in nature. Hence, in order to guarantee convergence, we modify the EVI algorithm by applying an aperiodicity transformation (as in \cite[Section 8.5.4]{puterman2014markov} ) \eqref{eqn:aperiodicity_transform}.

Thus at each iteration, Modified-EVI (Algorithm~\eqref{alg:MEVI}) applies a self transition probability of $(1-\tau)$, where $0<\tau<1$, to the same (state,period index) pair. As shown in \cite[Proposition 8.5.8]{puterman2014markov}, this transformation does not affect the average reward of any stationary policy.

\begin{algorithm}[H]
\caption{Modified - EVI}\label{alg:MEVI}
\begin{algorithmic}
\STATE \textbf{Input:} $ \mathcal{M}_k, \epsilon =  1/\sqrt{t_k}$\\
\STATE \textbf{Initialization:} $u_0(s,n) = 0 \forall s,n,s^{*} \in \mathcal{S}, n^{*} \in \{1,...N\}$ \\
\FOR{i = 0,1,2,...} 
\vspace{-5mm}
\STATE 
\footnotesize
\begin{equation} \label{eqn:aperiodicity_transform}
\begin{split}
u_{i+1}(s,n) & = \max_{a \in \mathcal{A}} \{ \max_{\dot{r} \in \mathcal{R}((s,n),a)} {\dot{r}((s,n),a)} \\
 & + \tau * \max_{\dot{p}  \in \mathcal{P}((s,n),a)} \{\sum_{s'} u_i(s',n+1)\dot{p}(s'|(s,n),a)\} \\
 & + (1-\tau) * u_i(s,n) 
\end{split}
\end{equation}

\STATE $u_{i+1}(s,n)  = u_{i+1}(s,n) - u_{i+1}(s^{*},n^{*})$
\vspace{+2mm}
\IF{ $\max_{(s,n)} \{ u_{i+1}(n,s) - u_{i}(n,s)\} - \min_{(s,n)} \{u_{i+1}(n,s) - u_{i}(n,s)\}  \leq \epsilon $}
\normalsize
\STATE Break the for loop.
\ENDIF
\ENDFOR
\end{algorithmic}
\end{algorithm}

\subsection{Analysis}

Let $T((s',n')| M, (s,1))$ denote the expected first hitting time of  $(s',n')$ of an AMDP $M$, starting from $(s,1)$ under a stationary policy $\pi : \mathcal{S} \times \{1,2,....,N\} \xrightarrow{} \mathcal{A}$ . As in \cite[Definition 1]{auer2008near} the diameter of an AMDP $M$ is defined as:

\footnotesize
\begin{equation} \label{eqn:diameter}
    D_{aug} = \max _{(s',n')\neq (s,1),(s',s) \in \mathcal{S} }\min_{\pi} \mathds{E}[T((s',n')| M, (s,n))]
\end{equation}
\normalsize
\newtheorem{1}{}
\begin{theorem} \label{theorem:regret_bound}
With probability at least $1-\delta$,  the regret for PUCRL2 is:
\footnotesize
\begin{equation*}
\Delta(PUCRL2) \leq 34D_{aug} SN\sqrt{AT \log \frac{T}{\delta}}
\end{equation*}
\end{theorem}
\normalsize
\begin{proof}
See Appendix \eqref{sec:proof_theorem_1}. 
\end{proof}

\begin{remark}
The confidence bound \eqref{eqn:confidence_bound_p} ignores the known sparsity in the transition function. If we include that knowledge, we obtain the same regret bound. However, when implementing this case Modified-EVI does not converge for few iterations. This issue is left as open work for now.
\end{remark}

\section{Numerical results}
We compare the performance of PUCRL2 with three other algorithms: (i) UCRL2 \cite{auer2008near} which provides optimal static regret in stationary MDP setting, (ii) UCRL3 \cite{bourel2020tightening} which is a recent improvement over UCRL2, and (iii) BORL \cite{cheung2020reinforcement} which is a   parameter free algorithm for the non-stationary setting.

\subsection{Regret of BORL for PMDP}
The variation budget as in \cite{cheung2020reinforcement} for the rewards is defined as: 
\footnotesize
 \begin{equation*}
    B_r = \sum_{t=1}^{T-1} \max_{s \in \mathcal{S}, a \in \mathcal{A}}| r_{t+1}(s,a) - r_t(s,a) |
 \end{equation*}
 
 \normalsize
For a PMDP:
\vspace{-1mm}
\footnotesize
\begin{equation*} 
\begin{split}
B_r & = \sum_{t=1}^{T-1} \max_{s \in \mathcal{S}, a \in \mathcal{A}}| r_{t+1}(s,a) - r_t(s,a) |  \\
 &  \approx (T/N) \sum_{t=1}^N \max_{s \in \mathcal{S}, a \in \mathcal{A}}| r_{t+1}(s,a) - r_t(s,a) | \approx \mathcal{\Tilde{O}}(T)
\end{split}
\end{equation*}

\normalsize

Regret bounds of BORL and SW-UCRL \cite{cheung2020reinforcement} for non-stationary MDP are derived in terms of the reward variation budget $B_r$ and a very similar variation budget $B_p$ on the transition kernels. However, for a PMDP, these two algorithms do not exploit the additional structure arising out of periodicity. Since $B_r$ or $B_p$ turn out to be of the order $\mathcal{\Tilde{O}}(T)$ , the $\mathcal{\Tilde{O}}((B_r+B_p)^{1/4}T^{3/4})$ regret bound of  BORL or SW-UCRL becomes $\mathcal{\Tilde{O}}(T)$ for PMDP.

\begin{figure} 
  \centering
  \includegraphics[scale=0.5]{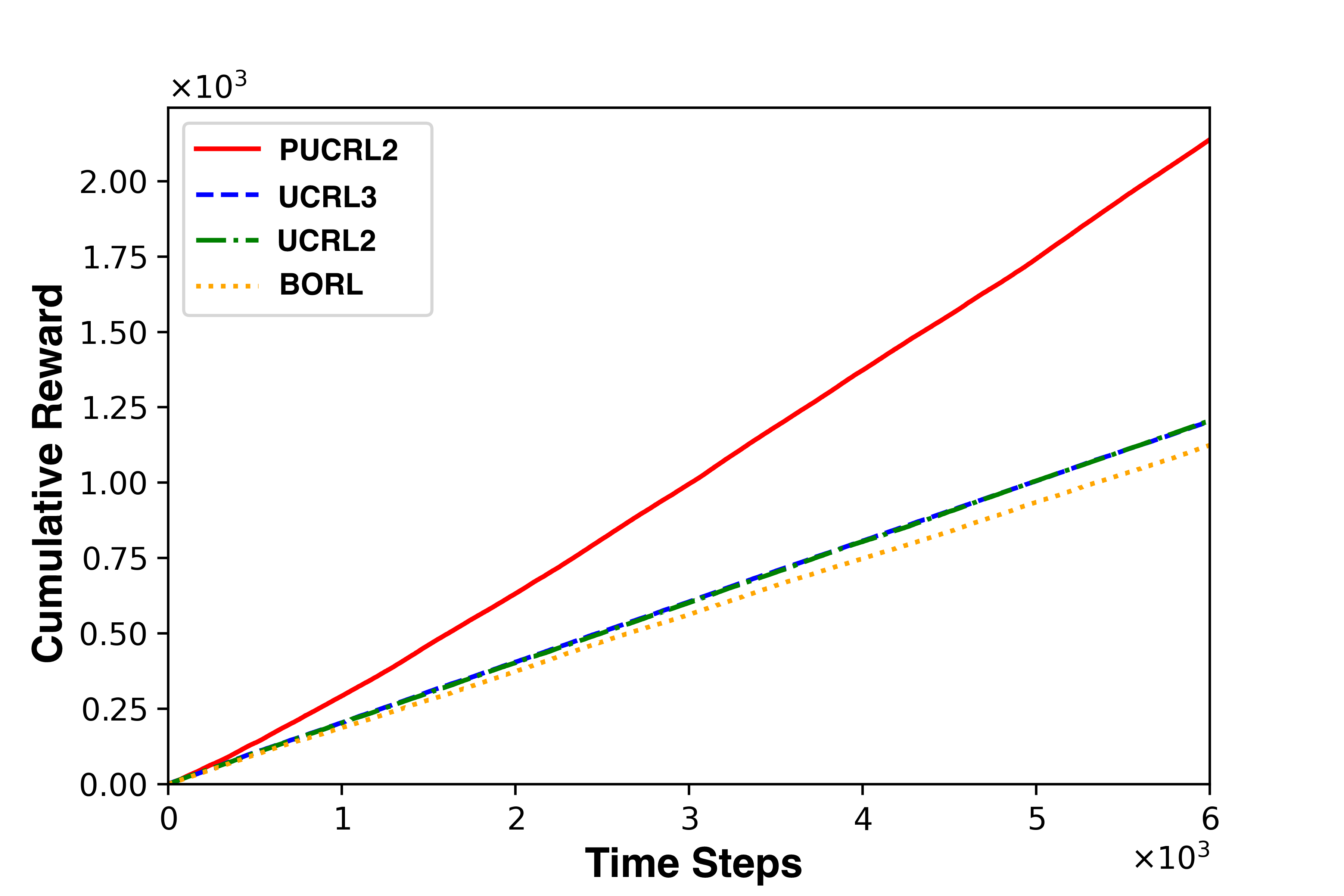}
  \includegraphics[scale=0.5]{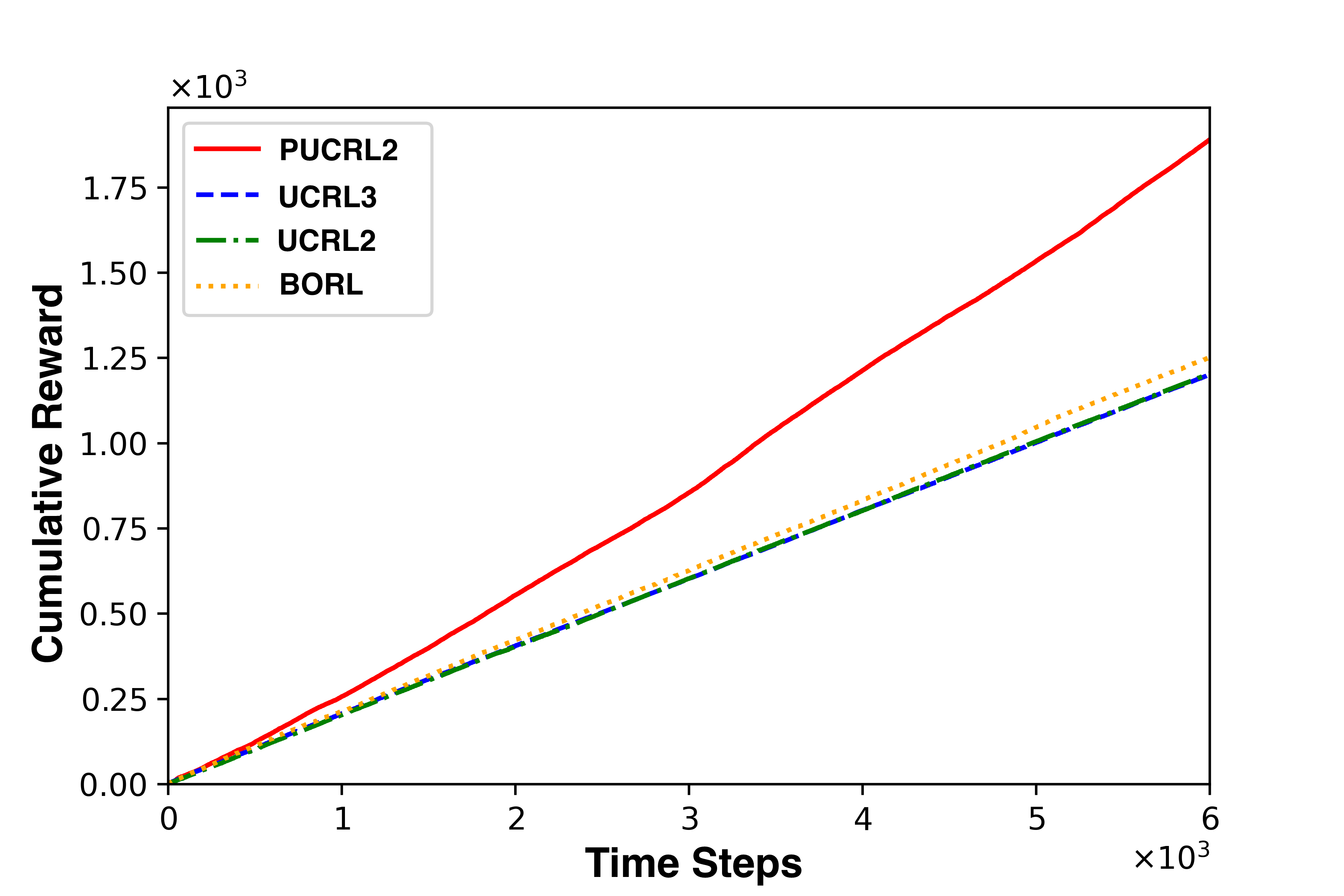}
  \caption{Cumulative reward for a  2-state, 2-action PMDP with N = 5(Above) and  N = 15(Below).}
  \label{fig:Reward_plot}
\end{figure}

\subsection{Our experiment}
Our synthetic data-set formulation is inspired by \cite{cheung2020reinforcement}. We consider a MDP with two states $\{s_1, s_2\}$, two actions  $\{a_1, a_2\}$ and $T=  6000$. The variation in the rewards and transition function are modeled using   cosine functions  as follows:

\footnotesize
\begin{equation*} 
\begin{split}
r_t(s_1,a_1) & = 0.2+0.3 \cos (2\pi t/N ), r_t(s_1,a_2) = 0.2+\cos (2\pi t/N ) \\
r_t(s_2,a_1) & = 0.2-\cos (2\pi t/N ), r_t(s_2,a_2) = 0.2-0.3\cos (2\pi t/N )
\end{split}
\end{equation*}
\normalsize
and 
\footnotesize
\begin{eqnarray*}
p_t(s_1|s_1,a_1) &=& 1,  p_t(s_2|s_1,a_1) = 0,\\
p_t(s_1|s_1,a_2) &=& 1-\beta_t,  p_t(s_2|s_1, a_2) = \beta_t, \\
p_t(s_1|s_2,a_1) &=& 0,  p_t(s_2|s_2,a_1) = 1,  \\
p_t(s_1|s_2,a_2) &=& \beta_t,  p_t(s_2|s_2, a_2) = 1-\beta_t
\end{eqnarray*}

\normalsize
where,
$\beta_t =0.5+0.3 \sin (5V_p \pi t/N )$. We set the period $N = 5 $ and $15$, $\delta=0.05$, and  compare the cumulative reward of the algorithms after averaging   over
$30$ independent runs. The results are shown in Figure \eqref{fig:Reward_plot}.
We clearly observe that PUCRL2 outperforms other algorithms.

\section{Conclusion}
Periodic non-stationarity in Markov
Decision Processes has been studied in this paper, where the state transition and reward
functions vary periodically. Existing RL algorithms for non-stationary and stationary MDPs fail to perform optimally in this setting. We provide a new algorithm called PUCRL2, which outperforms competing algorithms in the field. The static regret term depends linearly on the diameter of the AMDP, the comparison of which with the maximum diameter of non-stationary MDPs is left as our future work.

{\small
\bibliographystyle{unsrt}
\bibliography{reference.bib}
}

\newpage
 \appendices

 \section{PROOF OF THEOREM 1}
 \label{sec:proof_theorem_1}

The proof borrows some ideas from \cite{auer2008near} and is divided into sections. 
In Appendix \eqref{subsec:splitting_into_episodes}, we upper bound the total regret by removing the randomness in the rewards accumulated. The regret in the episodes where the true AMDP does not lie in the set of plausible AMDPs is bounded above in Appendix \eqref{subsec:failing_confidence_region}, and with the assumption that it does in Appendix \eqref{subsec:M_in_M_k}. Finally, we complete the proof in Appendix \eqref{subsec:completing_the _proof}.

\subsection{Splitting into episodes} \label{subsec:splitting_into_episodes}

As in \cite[Section 4.1]{auer2008near} using Hoeffding’s inequality , we can decompose the regret as:
\footnotesize
\begin{eqnarray*} \label{eqn2}
\Delta &=& \sum_{t =1}^{T} (\rho^* - r_t((s_t,n_t),a_t)) \\
  &\leq& T\rho^* - \sum_{(s,n),a} N((s,n),a) r((s,n),a) + \sqrt{\frac{5}{8}T \log \frac{8T}{\delta}}  
\end{eqnarray*}
\normalsize
with probability at least $1-\frac{\delta}{12 T^{5/4}}$
, where $N((s,n),a)$ is the count of (state,period)-action pair after $T$ steps.

Let there be m episodes in total , thus
$\sum_{k=1}^m v_k((s,n),a) = N((s,n),a)$.

The regret in each episode can be defined as : $\Delta_k = \sum_{(s,n),a} v_k((s,n),a) (\rho^* - r((s,n),a))$. Hence,
\footnotesize
\begin{equation} \label{eqn:regret_decomposition}
\Delta\leq \sum_{k=1}^m \Delta_k + \sqrt{\frac{5}{8}T \log \frac{8T}{\delta}}
\end{equation}
\normalsize

\subsection{ Dealing with failing confidence regions} \label{subsec:failing_confidence_region}

\begin{lemma} \label{lemma:probability-of-mdp-outside-confidence-region}
For any $t \geq 1$, the probability that the true AMDP M is not contained in the set of plausible
AMDPs $\mathcal{M}(t)$ at time t is at most $\delta / 15 t^6$, that
is 
\begin{equation*}
\mathbb{P}\{M \notin \mathcal{M}(t)\} < \delta / 15t^6
\end{equation*}

\end{lemma}
\begin{proof}
As in \cite[Section C.1]{auer2008near} we bound the transition functions using $L^1$-deviation concentration inequality over $m$ distinct events from $l$ samples \cite{weissman2003inequalities}:
\begin{eqnarray*}
\mathbb{P} \{ \lVert  \hat{\mathbf{p}}(\cdot) - \mathbf{p(\cdot)} \rVert_1 \geq \epsilon_p \} \leq (2^m-2) \exp(-l\epsilon_p^2 / 2)
\end{eqnarray*}

As the state space has been augmented, we have $SN$ states and hence $m = SN$ events. 
\newline
Thus, setting 
\footnotesize
\begin{eqnarray*}
\epsilon_p = \sqrt{\frac{2}{l} \log(\frac{2^{SN}20SAt^7}{\delta})} \leq \sqrt{\frac{14SN}{l} \log(\frac{2At}{\delta})}
\end{eqnarray*}
\normalsize
we get,
\footnotesize
\begin{eqnarray*}
\mathbb{P} \{ \lVert  \hat{\mathbf{p}}(\cdot|(s,n),a) - \mathbf{p}(\cdot|(s,n),a) \rVert_1 \geq \sqrt{\frac{14SN \log(2At / \delta)}{l}} \} \leq \frac{\delta}{20t^7SA}
\end{eqnarray*}
\normalsize
For rewards, we use Hoeffding's inequality to bound the deviation of empirical mean from true mean given $l$ i.i.d samples
\footnotesize
\begin{eqnarray*}
\mathbb{P} \{ \lvert  \hat{r} - r \rvert \geq \epsilon_r\} \leq 2 \exp(-2l\epsilon_r^2)
\end{eqnarray*}
\normalsize
Setting 
\footnotesize
\begin{eqnarray*}
\epsilon_r = \sqrt{\frac{1}{2l} \log(\frac{120SAt^7}{\delta})} \leq \sqrt{\frac{7}{2l} \log(\frac{2SAt}{\delta})}
\end{eqnarray*}
\normalsize
we get for all $((s,n),a)$ pair
\footnotesize
\begin{eqnarray*}
\mathbb{P} \{ |  \hat{r}((s,n),a) - r((s,n),a) | \geq \sqrt{\frac{7 \log(2SAt / \delta)}{2l}} \} \leq \frac{\delta}{60t^7SA}
\end{eqnarray*}

\normalsize
 A union bound over all possible values of $l$ i.e. $l$ = 1,2,..... $\lfloor t/N \rfloor$ ,gives ($n_k((s,n),a)$ denotes the number of visits in $((s,n),a)$)
\footnotesize

\begin{align}
\mathbb{P} \{ \lVert \hat{\mathbf{p}}(\cdot|(s,n),a) & - \mathbf{p}(\cdot|(s,n),a) \rVert_1 \geq \sqrt{\frac{14SN \log(2At / \delta)}{n_k((s,n),a)}} \} \nonumber \\
&\leq \sum_{t=1} ^{\lfloor t/N \rfloor} \frac{\delta}{20t^7SA} \nonumber \leq  \sum_{t=1} ^{t/N}  \frac{\delta}{20t^7SA}= \frac{\delta}{20t^6SAN} \nonumber
\end{align}
\begin{eqnarray*}
\mathbb{P} \{ |  \hat{r}((s,n),a) - r((s,n),a) | \geq \sqrt{\frac{7 \log(2SAt / \delta)}{2n_k((s,n),a)}} \} \leq \sum_{t=1} ^{\lfloor t/N \rfloor} \frac{\delta}{60t^7SA} \\
\leq \sum_{t=1} ^{t/N}  \frac{\delta}{60t^7SA}= \frac{\delta}{60t^6SAN}
\end{eqnarray*}

\normalsize

Summing these probabilities over all (state,period)-action pairs we obtain the claimed bound 
$\mathbb{P}\{M \notin \mathcal{M}(t)\} < \delta / 15t^6$.

\end{proof}

\begin{lemma} \label{lemma:regret_due_to_failing_confidence_region}
With probability at least $1-\frac{\delta}{12T^{5/4} }$, the regret occurred due to failing confidence region i.e. 
\begin{equation} \label{eqn:M_not_in_confidence_bound}
\sum_{k=1}^m \Delta_k \mathds{1}_{M \notin \mathcal{M}_k} \leq \sqrt{T}
\end{equation}
\end{lemma}
\begin{proof}
Refer \cite[Section 4.2]{auer2008near} with Lemma \eqref{lemma:probability-of-mdp-outside-confidence-region} instead of \cite[Appendix C.1]{auer2008near}
\end{proof}

\subsection{Episodes with $M \in \mathcal{M}_k$} \label{subsec:M_in_M_k}

By the assumption $M \in \mathcal{M}_k$ and \cite[Theorem 7]{auer2008near}, the optimistic optimal average reward of the near optimal policy $\tilde{\pi}_k$ chosen in Modified-EVI \eqref{alg:MEVI} is such that $\Tilde{\rho_k} \geq \rho^* - 1/\sqrt{t_k}$.

Thus, we can write the regret of an episode as :

\footnotesize
\begin{equation} \label{eqn:epsiodic_regret_optimism}
\begin{split}
\Delta_k & = \sum_{(s,n),a} v_k((s,n),a) (\rho^* - r((s,n),a))  \\
 & \leq \sum_{(s,n),a} v_k((s,n),a) (\tilde{\rho}_k - r((s,n),a)) + \sum_{(s,n),a} \frac{v_k((s,n),a)}{\sqrt{t_k}}. 
\end{split}
\end{equation}

\normalsize
Let us define $i_k$ to be the last iteration when convergence criteria holds and Modified-EVI terminates, thus as in \cite[Section 4.3.1]{auer2008near}

\footnotesize
\begin{equation} \label{eqn:EVI_condition}
|u_{i_k+1}(s,n) - u_{i_k}(s,n) - \tilde{\rho_k}| \leq 1/\sqrt{t_k}  
\end{equation}
\normalsize
for all $(s,n)$. Expanding as in \eqref{eqn:aperiodicity_transform}
\footnotesize
\begin{eqnarray*} 
u_{i_k+1}(s,n) &=& \tilde{r}_k((s,n),\tilde{\pi}_k(s,n)) \\
 &+& \tau * \{\sum_{s'} u_{i_k}(s',n+1)\tilde{p}_k(s'|(s,n),\tilde{\pi}_k(s,n))\} \\
 &+& (1-\tau) * u_{i_k}(s,n)\}
\end{eqnarray*}

\normalsize

Putting it in \eqref{eqn:EVI_condition}, we get

\footnotesize
\begin{eqnarray*}  
|\tilde{\rho_k} - \tilde{r}_k((s,n),\tilde{\pi}_k(s,n)) 
- \tau * \{\sum_{s'} u_{i_k}(s',n+1)\tilde{p}_k(s'|(s,n),\tilde{\pi}_k(s,n))\} \\
- (\cancel{1}-\tau) * u_{i_k}(s,n)\ + \cancel{u_{i_k}(s,n)} | \leq 1/\sqrt{t_k} 
\end{eqnarray*}
\vspace{-5mm}
\begin{equation*} 
\begin{split}
\tilde{\rho_k}  - \tilde{r}_k((s,n),\tilde{\pi}_k(s,n)) & \leq  \tau * \{\sum_{s'} u_{i_k}(s',n+1) \\
 & \tilde{p}_k(s'|(s,n),\tilde{\pi}_k(s,n))\} - \tau * u_{i_k}(s,n) + 1/\sqrt{t_k}
\end{split}
\end{equation*}

\normalsize 
Thus, putting the above result in \eqref{eqn:epsiodic_regret_optimism}, and noting that $\sum_{(s,n),a} v_k((s,n),a) = 0$, for $a \neq \tilde{\pi}_k(s,n)$, we get

\footnotesize
\begin{equation} \label{eqn:del_p_del_r}
\begin{split}
\Delta_{k} & \leq \underbrace{\tau \sum_{(s,n),a} v_k((s,n),a) ( \sum_{s'} u_{i_k}(s',n+1)\tilde{p}_k(s'|(s,n),a)  -  u_{i_k}(s,n))}_{ \coloneqq \Delta_{k}^{p}} \\
& + \underbrace{\sum_{(s,n),a} v_k((s,n),a) (\tilde{r}_k((s,n),a)) - r((s,n),a))}_{\coloneqq \Delta_{k}^{r}} \\
& + 2 \sum_{(s,n),a} \frac{v_k((s,n),a)}{\sqrt{t_k}}  
\end{split}
\end{equation}

\normalsize

\subsubsection{Bounding $\Delta_k^p$}

\footnotesize

\begin{equation} \label{eqn:del_p}
\begin{split}
\Delta_k^p & =  \tau\sum_{(s,n),a} v_k((s,n),a)( \{\sum_{s'} u_{i_k}(s',n+1)\tilde{p}_k(s'|(s,n),a)\} 
\\ & -  u_{i_k}(s,n))) \\
& = \tau\sum_{(s,n),a} v_k((s,n),a)( \sum_{s'} u_{i_k}(s',n+1) \\
& (\tilde{p}_k(s'|(s,n),a) - p_k(s'|(s,n),a)) + \tau\sum_{(s,n),a} v_k((s,n),a) 
\\  & ( \sum_{s'}u_{i_k}(s',n+1) p_k(s'|(s,n),a) - u_{i_k}(s,n)) 
\end{split}
\end{equation}

\normalsize
By the property of extended value iteration\cite[Section 4.3.1]{auer2008near}, extended to Modified-EVI
\begin{equation} \label{eqn:span_bound}
     span(\mathbf{u}_{i_k}) = \max_{(s,n)} u_{i_k}(s,n) - \min_{(s,n)} u_{i_k}(s,n) \leq D_{aug}^\tau  \\
\end{equation} 
where $D_{aug}^\tau$ represents  the diameter of the augmented MDP with aperiodicity transformation.

Since, $ \sum_{s'} p_k(s'|(s,n),a) =1$ and $\sum_{s'}\tilde{p}_k(s'|(s,n),a)=1$, we can replace $u_{i_k}(s,n)$ by 

\footnotesize
\begin{equation} 
    w_k(s,n) = u_{i_k}(s,n) - \frac{\max_{(s,n)} u_{i_k}(s,n) + \min_{(s,n)} u_{i_k}(s,n)}{2}
\end{equation}
\normalsize
such that it follows from \eqref{eqn:span_bound} that $span(\mathbf{u}_{i_k}) = span (\mathbf{w}_k$).

Hence, $\lVert \mathbf{w}_k \rVert_\infty \leq D_{aug}^\tau/2$.

According to \cite[Section 3.3.1]{fruit2019exploration}, $D_{aug}^\tau  \leq D_{aug}/\tau$. Hence,
$\lVert \mathbf{w}_k \rVert_\infty \leq D_{aug}/2\tau$.

Thus, the first term in \eqref{eqn:del_p} can be bounded as :
 \footnotesize
\begin{equation*} 
{\tau}\sum_{(s,n),a} v_k((s,n),a)( \sum_{s'} w_k(s',n+1)(\tilde{p}_k(s'|(s,n),a) - p_k(s'|(s,n),a))
\end{equation*}
\begin{align*} 
\leq \tau\sum_{(s,n),a} v_k((s,n),a)(  \lVert \mathbf{w}_k \rVert_\infty  \lVert  \mathbf{\hat{p}}_k(\cdot|(s,n),a) - \mathbf{p}(\cdot|(s,n),a) \rVert_1)
\end{align*}

\begin{align} \label{eqn:del_p_first_term}
\leq \sum_{(s,n),a} v_k((s,n),a) \cancel{2\tau}\sum_{(s,n),a} \sqrt{\frac{14SN \log(2At_k / \delta)}{n_k((s,n),a)}} D_{aug}/\cancel{2\tau}
\end{align}
\normalsize
where the last inequality uses the confidence bound \eqref{eqn:confidence_bound_p}. We note that the aperiodicity transformation coefficient gets canceled out and does not appear in the regret term.

Following the proof of \cite[Second term, Section 4.3.2]{auer2008near}, the second term in \eqref{eqn:del_p} can be bounded as:

\footnotesize
\begin{equation} \label{eqn:del_p_second_term}
\begin{split}
\tau \sum_{k=1}^m \sum_{(s,n),a} v_k((s,n),a) & ( \sum_{s'} u_i(s',n+1) p_k(s'|(s,n),a) - u_i(s,n))\\
& \leq \tau\ D_{aug}^\tau \sqrt{\frac{5}{2}T \log \frac{8T}{\delta}} + m \tau\ D_{aug}^\tau \\
& \leq \cancel{\tau}\ D_{aug}/\cancel{\tau} \sqrt{\frac{5}{2}T \log \frac{8T}{\delta}} + m \cancel{\tau} D_{aug}/\cancel{\tau} 
\end{split}
\end{equation}
\normalsize
with probability at least $1-\frac{\delta}{12T^{5/4}}$ ,where $m \leq SNA \log \frac{8T}{SNA}$ is the number of episodes as in \cite[Appendix C.2]{auer2008near}.
\vspace{+2mm}
\subsubsection{Bounding $\Delta_k^r$}
\footnotesize
\begin{equation} \label{eqn:del_r}
\begin{split}
\Delta_k^r & = \sum_{(s,n),a} v_k((s,n),a) (\tilde{r}((s,n),a)) - r((s,n),a)) \\
& \leq \sum_{(s,n),a} v_k((s,n),a) (|\tilde{r}((s,n),a)) - \hat{r}((s,n),a))| \\
& + |\hat{r}((s,n),a)) - r((s,n),a))|) 
\\  & \leq 2\sum_{(s,n),a} v_k((s,n),a) \sqrt{\frac{7 \log(2SAt_k / \delta)}{2 n_k((s,n),a)}}   
\end{split}
\end{equation}

\normalsize
where the last inequality uses the confidence bound \eqref{eqn:confidence_bound_r}.

\subsection{Completing the Proof} \label{subsec:completing_the _proof}

Thus, we can write the total episodic regret using \eqref{eqn:del_p_del_r}, \eqref{eqn:del_p_first_term},\eqref{eqn:del_p_second_term}, and \eqref{eqn:del_r}, with probability at least $1-\frac{\delta}{12T^{5/4}}$:

\footnotesize
\begin{equation*} 
\begin{split}
\sum_{k=1}^m \Delta_k \mathds{1}_{M \in \mathcal{M}_k} & \leq \sum_{k=1}^m  \sum_{(s,n),a} v_k((s,n),a)D_{aug} \sqrt{\frac{14SN \log(2At_k / \delta)}{n_k((s,n),a)}} \\
& + D_{aug}\sqrt{\frac{5}{2}T \log \frac{8T}{\delta}} + D_{aug}  SNA \log \frac{8T}{SNA}
\\  & + (\sqrt{14 \log(2SAt_k / \delta)} + 2) \sum_{k=1}^m \sum_{(s,n),a} \frac{v_k((s,n),a)}{\sqrt{n_k((s,n),a)}} 
\end{split}
\end{equation*}

\normalsize
We can bound the term $ \sum_{k=1}^m \sum_{(s,n),a} \frac{v_k((s,n),a)}{\sqrt{n_k((s,n),a)}} \leq (\sqrt{2}+1)(\sqrt{SNAT})$ as in \cite[Section 4.3.3]{auer2008near}. Also, noting that $n_k((s,n),a)\leq t_k\leq T$.Thus,

\footnotesize
\begin{equation} \label{eqn:episodic_regret_bound}
\begin{split}
\sum_{k=1}^m \Delta_k \mathds{1}_{M \in \mathcal{M}_k} & \leq  D_{aug}\sqrt{\frac{5}{2}T \log \frac{8T}{\delta}} +   D_{aug}  SNA \log \frac{8T}{SNA} \\
& +(2D_{aug}\sqrt{14SN  \log(2AT / \delta)} + 2) (\sqrt{2}+1)(\sqrt{SNAT})
\end{split}
\end{equation}
\normalsize

Using \eqref{eqn:regret_decomposition}, \eqref{eqn:M_not_in_confidence_bound},  \eqref{eqn:episodic_regret_bound},  with a probability of $1-\frac{\delta}{4T^{5/4}}$, we can bound the total regret as:

\footnotesize
\begin{equation*} 
\begin{split}
\Delta & \leq \sum_{k=1}^m \Delta_k \mathds{1}_{M \in \mathcal{M}_k} + \sum_{k=1}^m \Delta_k \mathds{1}_{M \notin \mathcal{M}_k} + \sqrt{\frac{5}{8}T \log \frac{8T}{\delta}}  \\
& \leq  D_{aug}\sqrt{\frac{5}{2}T \log \frac{8T}{\delta}} + D_{aug}  SNA \log \frac{8T}{SNA} + (2D_{aug}\\
&\sqrt{14SN  \log(\frac{2AT}{\delta})}  + 2) (\sqrt{2}+1)(\sqrt{SNAT}) +\sqrt{T} + \sqrt{\frac{5}{8}T \log \frac{8T}{\delta}}
\end{split}
\end{equation*}

\normalsize

Further simplifications as in \cite[Appendix C.4]{auer2008near} yield the total regret as :
\footnotesize
\begin{equation*}
 \Delta \leq 34D_{aug}SN  \sqrt{AT\log(T / \delta)}
\end{equation*}
\normalsize
with a probability of $1 - \sum_{T=2}^\infty \frac{\delta}{4T^{5/4}} < 1 - \delta$ by union over all values of $T$.

\end{document}